\def\year{2020}\relax
\documentclass[letterpaper]{article} 
\usepackage{aaai20}  
\usepackage{times}  
\usepackage{helvet} 
\usepackage{courier}  
\usepackage[hyphens]{url}  
\usepackage{graphicx} 
\usepackage{subfigure}
\usepackage{booktabs}
\usepackage{graphicx}  
\usepackage{amsmath, amssymb}
\usepackage{amsthm}
\usepackage{xcolor}
\usepackage{bm}
\usepackage{nicefrac}

\DeclareMathOperator{\subjectto}{s.t.}
\DeclareMathOperator{\tovec}{vec}
\usepackage{array}
\newcolumntype{H}{>{\setbox0=\hbox\bgroup}c<{\egroup}@{}}

\newtheorem{lemma}{Lemma}
\newtheorem{proposition}{Proposition}
\newtheorem{theorem}{Theorem}
\newtheorem{definition}[theorem]{Definition}

\newcommand{\numitems}{k}
\newcommand{\permutations}{S_\numitems}
\newcommand{\plrebar}{\text{PL-REBAR}}
\newcommand{\plrelax}{\text{PL-RELAX}}
\newcommand{\scorefunction}{f}
\newcommand{\critic}{c}
\newcommand{\relaxcritic}{c_\phi}
\newcommand{\categoricalsamplerfunction}{H}
\newcommand{\plsamplerfunction}{H}
\newcommand{\uniform}{\operatorname{uniform}[0, 1]}
\newcommand{\gumbel}{\mathcal{G}}
\newcommand{\gumbelpdf}[1]{\phi_{#1}}
\newcommand{\gumbelcdf}[1]{\Phi_{#1}}
\newcommand{\continuousrv}{b_{\textit{cont}}}
\newcommand{\categoricalrv}{b}
\newcommand{\plrv}{b}
\newcommand{\gumbelrv}{z}
\newcommand{\uniformrv}{v}
\newcommand{\conditionalgumbelrv}{\tilde{z}}
\newcommand{\score}{\theta}
\newcommand{\scorecumsum}{\Theta}
\newcommand{\reparametrizationestimator}{\hat{g}_{\textit{reparam}}(\scorefunction)}
\newcommand{\someestimator}{\hat{g}_{\textit{CV}}(\scorefunction)}
\newcommand{\reinforceestimator}{\hat{g}_{\textit{REINFORCE}}(\scorefunction)}
\newcommand{\relaxestimator}{\hat{g}_{\textit{RELAX}}(\scorefunction)}
\newcommand{\rebarestimator}{\hat{g}_{\textit{REBAR}}(\scorefunction)}
\newcommand{\gumbelsoftmaxestimator}{\hat{g}_{\textit{Gumbel}}(\scorefunction)}

\newcommand{\citet}[1]{\citeauthor{#1} \shortcite{#1}}
\title{Low-variance Black-box Gradient Estimates for the Plackett-Luce Distribution}
\author{Artyom Gadetsky,\textsuperscript{\rm 1}\thanks{Both authors contributed equally to this work.}\thanks{Corresponding author. E-mail: artygadetsky@yandex.ru} Kirill Struminsky,\textsuperscript{\rm 1}\footnotemark[1] Christopher Robinson,\textsuperscript{\rm 2} \\
\Large \textbf{Novi Quadrianto,\textsuperscript{\rm 1, \rm 2} Dmitry Vetrov\textsuperscript{\rm 1\thanks{Samsung-HSE Laboratory}}}\\
\textsuperscript{\rm 1} National Research University Higher School of Economics \\
\textsuperscript{\rm 2} Predictive Analytics Lab (PAL), University of Sussex\\}
 
\begin{document}

\maketitle

\begin{abstract}
Learning models with discrete latent variables using stochastic gradient descent remains a challenge due to the high variance of gradient estimates. Modern variance reduction techniques mostly consider categorical distributions and have limited applicability when the number of possible outcomes becomes large. In this work, we consider models with latent permutations and propose control variates for the Plackett-Luce distribution. In particular, the control variates allow us to optimize black-box functions over permutations using stochastic gradient descent. To illustrate the approach, we consider a variety of causal structure learning tasks for continuous and discrete data. We show that our method outperforms competitive relaxation-based optimization methods and is also applicable to non-differentiable score functions.
\end{abstract}

\section{Introduction}
The vast majority of modern machine learning advancements share one central method - gradient-based optimization. Stochastic gradients give a scalable solution for learning, applicable when the loss function is too slow to compute due to the size of data or even intractable. The latter is often the case when the loss function includes an expectation over random latent variables. The objectives of this kind naturally arise in multiple settings, including probabilistic latent variable models \cite{neal1998view} and reinforcement learning \cite{williams1992simple}. Often the distribution of random variables also depends on the optimizable parameters of the loss function, which in turn makes gradient estimation harder and less reliable due to the high variance of stochastic gradients.

Despite the recent breakthroughs in gradient estimation for continuous latent variables \cite{kingma2013auto,rezende2014stochastic,mohamed2019monte}, gradient estimation for discrete latent variables remains a challenge. Currently, general-purpose estimators \cite{williams1992simple,mnih2014neural} remain unreliable and the state-of-the-art methods \cite{rebar,relax,yin2018arm} exclusively consider the categorical distribution. Although the reduction to the categorical case allows benefiting from gradient estimators for continuous relaxations, such solutions are hard to translate to discrete distributions with large support.

In this work, we consider a gradient estimator for the Plackett-Luce distribution, a distribution over permutations. Permutations naturally occur in various setting, such as ranking problems \cite{guiver2009bayesian}, optimal routing \cite{bello2016neural} and causal inference \cite{friedman2003}. However, the support of the distribution is superexponential in the number of items $\numitems$, which makes representing a distribution as a categorical distribution intractable even for dozens of items. At the same time, the Plackett-Luce distribution has $O(\numitems)$ parameters and allows sampling in $O(\numitems \log \numitems)$.

We translate the recent variance reduction techniques \cite{rebar,relax} to the case of Plackett-Luce distributions. Similarly to REBAR, we use the difference of the REINFORCE estimator and the reparametrized estimator for the relaxed model. In particular, we derive the conditional marginalization step \cite{rebar} for the Plackett-Luce case. In our experiments, we recast causal inference tasks as a variational optimization over permutations and solve it using a gradient optimization method. We show that our method outperforms competitive relaxation-based approaches for optimization over permutations \cite{urs,sinkhorn} for differentiable score functions and is applicable in a wider range of scenarios.

Our main contributions are the following:
\begin{itemize}
\item We derive a low-variance gradient estimator for the Plackett-Luce distribution.
\item We apply the gradient estimator to solve variational optimization tasks for black-box functions and concentrate primarily on causal inference tasks for continuous and discrete data.
\item For differentiable functions, we show that relaxation-based gradient optimization does not work out-of-the-box for causal inference tasks and propose additional constraints to achieve competitive results.
\end{itemize}

\section{A Brief Tour of Gradient Estimation}

We consider a general optimization task $\operatorname{min}_{\theta} \mathbb E_{p(\categoricalrv \mid \score)}[ \scorefunction(\categoricalrv)]$, where $\categoricalrv$ is a discrete random variable parametrized by $\score$. The expectation can be intractable, for instance when $b$ is a vector of categorical variables and the support of $b$ is exponential in the vector length. The standard solution is to construct a stochastic estimate for the gradient $\hat{g}(\scorefunction) := \frac{\partial}{\partial \score} \mathbb E_{p(\categoricalrv \mid \score)} [\scorefunction(\categoricalrv)]$ without explicitly computing the expectation. In this section, we briefly review the gradient estimation algorithms.

\subsection{REINFORCE}

The REINFORCE estimator \cite{williams1992simple} gives us a widely-applicable unbiased estimate for the gradient
\begin{equation}
\reinforceestimator = \scorefunction(\categoricalrv) \frac{\partial}{\partial \score} \log p(\categoricalrv \mid \score),\ \ \categoricalrv \sim p(\categoricalrv \mid \score).
\end{equation}
Although an unbiased gradient estimate is sufficient to guarantee convergence of stochastic gradient descent, in practice, the algorithm may not converge due to the high variance of the estimate \cite{rebar}. The variance of the REINFORCE estimator can be reduced using control variates. A Control variate is a function $\critic(\categoricalrv)$ with a zero mean $\mathbb E_{p(\categoricalrv \mid \score)} [\critic(b)] = 0$ that can be used to define another unbiased estimator
\begin{equation}
\someestimator = \reinforceestimator - \critic(\categoricalrv).
\end{equation}
The variance of the new estimator $\someestimator$ is lower than the variance of $\reinforceestimator$ if $\critic(\categoricalrv)$ is positively correlated with the random variable $\scorefunction(\categoricalrv)$. As an illustration, the gradient of probability $\frac{\partial}{\partial \score} \log p(\categoricalrv \mid \score)$ has zero mean, therefore it can be used as a control variate \cite{mnih2014neural}.

\subsection{Reparametrization Gradients for Continuous Relaxations}
The reparametrization trick \cite{kingma2013auto,rezende2014stochastic} is an alternative unbiased low-variance gradient estimator, applicable when $\scorefunction$ is differentiable and the latent variable $\continuousrv$ is continuous. The estimator represents the latent variable as a differentiable determinisitc transformation $\continuousrv = T(\uniformrv, \score)$ of a fixed distribution sample $\uniformrv$ and parameters $\score$ and estimates the gradient as
\begin{align}
& \reparametrizationestimator = \frac{\partial}{\partial \score} \scorefunction(\continuousrv) = \frac{\partial \scorefunction}{\partial T} \frac{\partial T}{\partial \theta}, \\
& \uniformrv_i \sim \uniform,\ i=1,\dots,\numitems.
\end{align}

Although the reparametrization trick is not applicable when the latent variable $\categoricalrv$ is discrete, \cite{jang2016categorical,maddison2016concrete} proposed the Gumbel-softmax estimator, a modification of the reparametrization trick for the relaxed categorical distribution.

To sample from a relaxed categorical distribution $p(\categoricalrv \mid \score)$ with probabilities $\tfrac{\exp \score_i}{\sum_j \exp \score_j}$, Gumbel-Softmax first samples a vector of independent Gumbel random variables $\gumbelrv_i \sim \gumbel(\score_i, 1), i=1,\dots,\numitems$
\begin{align}
& \gumbelrv_i = T(\score_i, \uniformrv_i) = \score_i - \log(-\log(\uniformrv_i))\\ 
& \uniformrv_i \sim \uniform,\ i=1,\dots,\numitems 
\end{align}
with location parameter $\theta$. According to the {\bf Gumbel-max trick} \cite{maddison2014sampling}, the index of the maximal element $H(z) = \operatorname{\arg}{\max} (z)$ is a categorical random variable with distribution $p(\categoricalrv \mid \score)$. Then, to make the sampler differentiable, the Gumbel-softmax trick replaces $\arg \max (z)$ with a relaxation $\operatorname{soft} \max(\gumbelrv) = \frac{1}{\sum \exp \gumbelrv_i} (\exp \gumbelrv_1, \dots, \exp \gumbelrv_\numitems)$. The gradient estimate is the reparametrization gradient for the relaxed categorical distribution:
\begin{align}
& \gumbelsoftmaxestimator = \frac{\partial}{\partial \score} \scorefunction(\categoricalrv) =\frac{\partial \scorefunction}{\partial \categoricalrv} \frac{\partial \categoricalrv }{\partial \gumbelrv} \frac{\partial \gumbelrv}{\partial \score}, \\
& \categoricalrv = \operatorname{soft}{\max}(z),\\
& \gumbelrv_i \sim \gumbel(\theta_i, 1), i=1,\dots,\numitems.
\end{align}
The resulting reparametrization gradient $\gumbelsoftmaxestimator$ has much lower variance than $\reinforceestimator$, but is generally biased due to the relaxation.

\subsection{Relaxation-based Control Variates}

Recently, \citet{rebar} and \citet{relax} proposed control variates for REINFORCE estimator based on the relaxed conditional distribution. Both works use the REINFORCE gradient estimator for the relaxed categorical distribution as a control variate for the non-relaxed estimator. To eliminate the bias of the REINFORCE estimator, they subtract the low-variance reparametrization gradient estimator. 

The key insight of \citet{rebar} is the conditional marginalization step used to correlate the non-relaxed REINFORCE estimator and the control variate. Importantly, the conditional marginalization relies on reparametrization trick for the conditional distribution $p(\gumbelrv \mid \categoricalrv, \score)$, obtained from the joint distribution $p(\categoricalrv, \gumbelrv \mid \score) = p(\categoricalrv | \gumbelrv) p(\gumbelrv \mid \score)$ of the Gumbel random vector $\gumbelrv$ and the output of the Gumbel-max trick $\categoricalrv = \categoricalsamplerfunction(\gumbelrv) = \operatorname{\arg}{\max} (\gumbelrv)$. \citet{rebar} derive a reparametrizable sampling scheme for $p(\gumbelrv \mid \categoricalrv, \score)$
\begin{equation} \label{categoricalconditionalgumbelsampler} 
\conditionalgumbelrv_i = \begin{cases}
                 -\log(-\log \uniformrv_i) & i = \categoricalrv \\
                 -\log\left( -\tfrac{\log \uniformrv_i}{\exp \score_i} + \exp(-\conditionalgumbelrv_\categoricalrv) \right) & i \neq \categoricalrv
                 \end{cases},
\end{equation}
where vector $\uniformrv$ is a uniform i.i.d. vector $\uniformrv \sim \uniform^\numitems$. This gives a two-step generative process for the distribution $p(\gumbelrv \mid \categoricalrv, \score)$. On the first step we sample the maximum variable $\uniformrv_\categoricalrv$ from the Gumbel distribtuion and on the second step we sample the other variables $\uniformrv_i, i \neq \categoricalrv$ from the Gumbel distribution trunctated  at $\conditionalgumbelrv_{\categoricalrv}$ with location parameter $\score_i$.

The unbiased RELAX estimator from \citet{relax} is
\begin{align} \label{relaxestimator}
\relaxestimator =& [\scorefunction(\categoricalrv) - \relaxcritic(\conditionalgumbelrv)]\frac{\partial}{\partial \score} \log p(\categoricalrv \mid \score) \nonumber \\
 & + \frac{\partial}{\partial \score} \relaxcritic(\gumbelrv) - \frac{\partial}{\partial \score} \relaxcritic(\conditionalgumbelrv) \\
 \categoricalrv = \categoricalsamplerfunction(\gumbelrv),\ \gumbelrv & \sim p(\gumbelrv \mid \score),\ \conditionalgumbelrv \sim p(\gumbelrv \mid \categoricalrv, \score)
\end{align}
where $\relaxcritic(\gumbelrv)$ is a parametric function optimized to reduce the variance of the estimator.

Similarly, for a differentiable function $\scorefunction$ the REBAR estimator by \citet{rebar} uses the function $\scorefunction$ with the relaxed argument $\operatorname{soft}\max(\gumbelrv)$ and tunes the scalar parameter $\eta$
\begin{align} \label{rebarestimator}
\rebarestimator =& [\scorefunction(\categoricalrv) - \eta \scorefunction(\operatorname{soft}\max(\conditionalgumbelrv))]\frac{\partial}{\partial \score} \log p(\categoricalrv \mid \score) \nonumber \\
 & + \eta \frac{\partial}{\partial \score} \scorefunction(\operatorname{soft}\max(\gumbelrv)) \nonumber \\ 
 & - \eta \frac{\partial}{\partial \score} \scorefunction(\operatorname{soft}\max(\conditionalgumbelrv)) \\
 \categoricalrv = \categoricalsamplerfunction(\gumbelrv),\ \gumbelrv & \sim p(\gumbelrv \mid \score),\ \conditionalgumbelrv \sim p(\gumbelrv \mid \categoricalrv, \score)
\end{align}

\section{Constructing Control Variates for the Plackett-Luce Distribution} \label{plpropositionsection}
In this paper, we extend the stochastic gradient estimators $\rebarestimator$ and $\relaxestimator$ from the categorical distribution to the Plackett-Luce distribution. With a slight abuse of notation, below we use letter $\plrv$ to denote an integer vector $\plrv = (\plrv_1,\dots,\plrv_\numitems) \in \permutations$ that represent a permutation, $\score$ to denote the parameters of the Plackett-Luce distribution and $p(\plrv \mid \score)$ to denote the Plackett-Luce distribution.

The goal of this section is to define the two components required to apply the aforementioned gradient estimators: the mapping $\categoricalrv = \plsamplerfunction(\gumbelrv)$ and the two  reparametrizable conditional distributions $p(\gumbelrv \mid \score)$ and $p(\gumbelrv | \plrv, \score)$. After this we apply the estimators as defined in eq.~\ref{relaxestimator} and eq.~\ref{rebarestimator}, but to emphasize the difference we refer to them as \plrelax \ and \plrebar.

 \begin{definition}
 The Plackett-Luce distribution \cite{luce2005individual,plackett1975analysis} with scores $\score = (\score_1, \dots, \score_\numitems)$ is a distribution over permutations $S_k$ with the probability of outcome $\plrv \in \permutations$
 \end{definition}
 \begin{equation}\label{plprob}
 p(\plrv | \score) = \prod_{j=1}^\numitems \frac{\exp \score_{\plrv_j}}{\sum_{u=j}^\numitems \exp \score_{\plrv_u}}.
 \end{equation}
 Intuitively, a sample from the Plackett-Luce distribution $\plrv = (\plrv_1,\dots,\plrv_\numitems)$ is generated as a sequence of samples from categorical distributions. The first component $\plrv_1$ comes from the categorical distribution with logits $\score$, then the second components $\plrv_2$ comes from the categorical distribution with the logits $\score$ without the component $\theta_{\plrv_1}$ and so on. 
 
 The Plackett-Luce can be used for variational optimization \cite{staines2012variational}. Indeed, at the lower temperatures $\score \rightarrow \tfrac{\score}{T}, T \ll 1$ the distribution converges to a divergent distribution. The mode of the Plackett-Luce distribution is the descending order permutation of the scores $\plrv^0: \score_{\plrv_1^0} \geq \dots \geq \score_{\plrv_\numitems^0}$, because $\plrv^0$ permutation maximizes each factor in the product in eq.~\ref{plprob}.
 
 Now we will give an alternative definition of the Plackett-Luce distribution. 
 \begin{lemma}\label{gumbelforpl} (appears in \cite{urs,yellott1977relationship}) Let $\gumbelrv$ be a vector of $\numitems$ independent Gumbel random variables with location parameters specified by score vector $\score$
 \begin{equation}
 \gumbelrv_i = \score_i - \log(-\log(\uniformrv_i)),\ \uniformrv_i \sim \uniform.
 \end{equation}
 Then for a permutation $\plrv \in \permutations$ the probability of event $\{\gumbelrv_{\plrv_1} \geq \dots \geq \gumbelrv_{\plrv_\numitems}\}$ is 
 \begin{equation}
 p(\gumbelrv_{\plrv_1} \geq \dots \geq \gumbelrv_{\plrv_\numitems}) = \prod_{j=1}^\numitems \frac{\exp \score_{\plrv_j}}{\sum_{u=j}^\numitems \exp \score_{\plrv_u}}.
 \end{equation}
 \end{lemma}
 Similarly to the {\bf Gumbel-max trick}, Lemma \ref{gumbelforpl} shows that an order of a Gumbel-distributed vector is distributed according to the Plackett-Luce distribution. Following the lemma, for Plackett-Luce distributions we define $p(\gumbelrv \mid \score)$ to be a Gumbel-distributed vector and $H(z)$ to be a sorting operation
 \begin{align}
 & \gumbelrv_i \sim \gumbel(\score_i, 1),\ i=1,\dots,\numitems \\
 & \plsamplerfunction(\gumbelrv) = \arg \operatorname{sort}(\gumbelrv)
 \end{align}
 
 Our principal discovery is that, similarly to the categorical case, the conditional distribution $p(\gumbelrv | \plrv, \score)$ factorizes into a sequence of truncated Gumbel distributions. As a consequence, the distribution is reparametrizable and can be used to construct a control variate for a gradient estimator.

 \begin{proposition} \label{plconditionalreparametrization}
  Let $p(\plrv, \gumbelrv \mid \score)$ be the joint distribution with $\gumbelrv_i \sim \gumbel(\score_i, 1)$, $\plrv = \arg \operatorname{sort} (\gumbelrv)$ and normalized parameters $\sum_{j=1}^\numitems \exp \score_j = 1$. Then for uniform i.i.d samples $\uniformrv_i \sim \uniform$ and $\scorecumsum_i = \sum_{j=i}^{\numitems} \exp \score_{\plrv_j}$ for $i=1,\dots,\numitems$ the vector $\conditionalgumbelrv = (\conditionalgumbelrv_1, \dots, \conditionalgumbelrv_\numitems)$
\begin{equation} \label{conditionalgumbelsampler}
\conditionalgumbelrv_{\plrv_i} = \begin{cases}
    - \log(-\log \uniformrv_{i}) & i=1 \\
    - \log(-\frac{\log \uniformrv_{i}}{\scorecumsum_i} + \exp (-\conditionalgumbelrv_{\plrv_{i-1}})) & i \geq 2,
\end{cases}
\end{equation}
 is a sample from the conditional distriubtion $p(\gumbelrv \mid \plrv, \score)$.
 
\end{proposition}

The proof of the proposition is given in the appendix.

The sampling procedure from Proposition~\ref{plconditionalreparametrization} has two principal differences from the sampling scheme for the categorical case (see eq. \ref{categoricalconditionalgumbelsampler}). First, the truncation parameter $\conditionalgumbelrv_{\plrv_{i - 1}}$ now depends on the previous component $i-1$, while for the categorical case the truncation parameter is defined by the maximum component. Second, the location parameter is now a cumulative sum and depends on the previous scores.

\begin{figure*}[ht!]
  \centering
  \includegraphics[width=\textwidth]{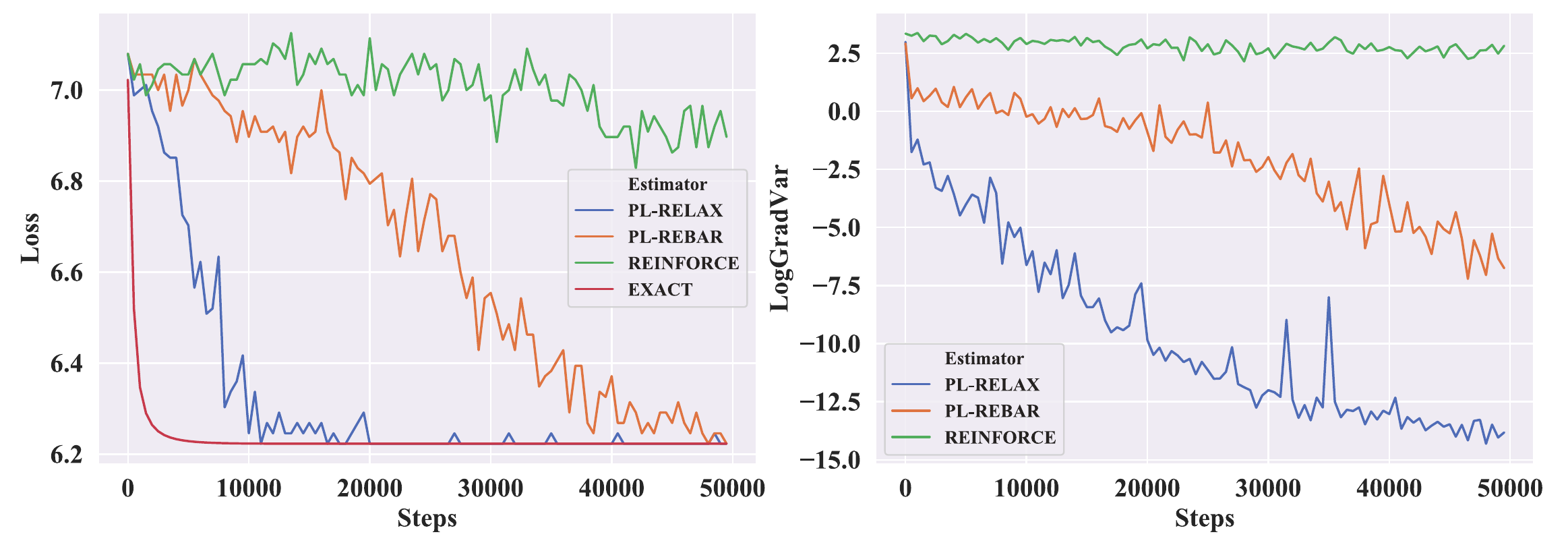}
  \caption{Training curves and log-variance of gradient estimators for different estimators on a toy problem: $\mathbb{E}_{p(b | \theta)} \|P_b - P_{0.05} \|_F^2$}
  \label{toyexp}
\end{figure*}
\section{Related Work}
\citet{jang2016categorical,maddison2016concrete} use the Gumbel distribution and Gumbel-max trick to define continuous relaxations of discrete distributions, by providing a gradient estimator which replaces the sampling of a categorical distribution with a differentiable sample from a Gumbel-Softmax distribution. 

The Gumbel-Softmax distribution does not scale to permutations, as distribution over $\numitems$-dimensional permutations is equivalent to that over $\numitems !$ categories. Recently, a line of work proposed various for optimization over permutations. \citet{birkhofflinderman} relaxes the discrete set of permutations to Birkhoff polytope, the set of doubly-stochastic matrices, and extend stick-breaking approach \cite{stickbreaking} to satisfy polytope constraints. \citet{sinkhorn} obtain doubly-stochastic matrices by applying the Sinkhorn operator. They use the Gumbel-Softmax distribution to define a distribution over latent matchings, the implicit Gumbel-Sinkhorn distribution. 
\citet{urs} define new relaxation to the set of unimodal row-stochastic matrices, the set of matrices that have a unique maximal element in every row.

\citet{relax} extend \citet{rebar} and derive control variate for black-box function optimization combining the REINFORCE estimator and reparametrization trick. \citet{yin2018arm} propose gradient estimator that estimates the gradients of discrete distribution parameters in an augmented space.

For the special case of TSP, \cite{bello2016neural,kool2018attention} introduce an amortized family of distributions over permutations using a deep autoregressive model and design control variates that exploit the structure of the loss function. 

\section{Experiments}
We demonstrate the effectiveness of the proposed method with a simple toy task similar to \citet{rebar} and then continue to the more challenging task of optimization over topological orderings for solving causal structure learning problems. Our PyTorch \cite{paszke2017automatic} implementation of the gradient estimators is available at https://github.com/agadetsky/pytorch-pl-variance-reduction .

\subsection{Toy Experiment}
As a proof of concept we perform an experiment in minimizing $\mathbb{E}_{p(b | \theta)} \|P_b - P_t \|_F^2 = \mathbb{E}_{p(b | \theta)} f(P_b)$ as a function of $\theta$ where $p(b|\theta)=$ Plackett-Luce$(b|\theta)$. $P_b$ is permutation matrix with elements $p_{i, b_i} = 1$ and $P_t$ is a matrix with $\frac{1}{\numitems} + t$ on the main diagonal and $\frac{1}{\numitems} - \frac{t}{\numitems-1}$ in the remaining positions. This problem can be seen as linear sum assignment problem with specifically constructed doubly stochastic matrix $P_t$. It is easy to note that taking $\numitems=2$ and $t=0.05$ leads to toy problem similar to that of \citet{rebar}. We focus on $t=0.05$ and $\numitems=8$ to enable computation of exact gradients. For the $\plrebar$ estimator we take $c_\phi(z) = \eta f(\sigma(z, \tau))$ where $\sigma(z, \tau)$ is the continuous relaxation of permutations described by \citet{urs}. For the $\plrelax$ estimator we take $c_\phi(z) = f(\sigma(z, \tau)) + \rho_\phi(z)$ where $\rho_\phi(z)$ is a simple neural network with two linear layers and ReLU activation between them. Figure \ref{toyexp} shows the relative performance and gradient log-variance of REINFORCE, $\plrebar$ and $\plrelax$. Although the REINFORCE estimator is unbiased, we can see that the variance of the estimator is too large even for the simple toy task, therefore the method is completely inapplicable for optimization over permutations. On the other hand, the proposed method significantly reduces variance of the gradient and thus converges to optimal. Also, similarly to the toy experiment from \citet{relax} paper, we observe better performance of the $\plrelax$ estimator due to free-form control variate parameterized by a neural network.

\begin{table*}[p]
\centering
\caption{Results for ER and SF graphs of 10 nodes}
\label{10nodes}
\resizebox{\textwidth}{!}{
\begin{tabular}{llHllllHlll}
\toprule
& \multicolumn{2}{l}{ER1} & & & & \multicolumn{2}{l}{ER4} & &  \\ \midrule
& \multicolumn{1}{l}{Val $\widehat{Q} - \widehat{Q}^*$} & & \multicolumn{1}{l}{SHD} & \multicolumn{1}{l}{SHD-CPDAG} & \multicolumn{1}{l}{SID}  & \multicolumn{1}{l}{Val $\widehat{Q} - \widehat{Q}^*$} & & \multicolumn{1}{l}{SHD} & \multicolumn{1}{l}{SHD-CPDAG} & \multicolumn{1}{l}{SID}  \\ \midrule
$\plrelax$ & -0.2$\pm$1.7 & -0.3$\pm$1.6 & 5.2$\pm$2.5 & 5.8$\pm$3.2 & 13.0$\pm$9.6 & 12.2$\pm$26.3 & 12.9$\pm$28.0 & 29.4$\pm$1.9 & 35.0$\pm$4.4 & 67.0$\pm$1.8 \\
SINKHORN$_{ECP}$     & 1.8$\pm$5.3 & 1.8$\pm$5.2 & 5.6$\pm$2.7 & 6.4$\pm$2.9 & 14.2$\pm$10.2 & 4.8$\pm$10.4 & 4.8$\pm$10.4 & 31.2$\pm$2.6 & 33.6$\pm$2.7 & 69.6$\pm$2.3   \\
URS$_{ECP}$    & 13.5$\pm$26.9 & 12.9$\pm$26.2 & 7.4$\pm$3.7 & 7.4$\pm$3.6 & 16.0$\pm$8.9 & 12.4$\pm$6.2 & 12.5$\pm$7.0 & 29.8$\pm$3.6 & 32.8$\pm$4.9 & 67.4$\pm$2.1   \\
SINKHORN & 85.9$\pm$101.2 & 81.0$\pm$96.8 & 12.0$\pm$3.7 & 12.0$\pm$3.7 & 29.4$\pm$17.3 & 4019.6$\pm$3138.0 & 3945.0$\pm$2957.4 & 36.6$\pm$2.4 & 37.8$\pm$1.7 & 79.8$\pm$6.9 \\
URS & 71.4$\pm$128.9 & 69.3$\pm$125.6 & 10.8$\pm$2.9 & 11.0$\pm$3.2 & 26.0$\pm$10.5 & 1894.9$\pm$1704.8 & 1821.5$\pm$1584.4 & 34.6$\pm$2.2 & 36.8$\pm$2.8 & 74.4$\pm$2.7 \\
GREEDY-SP & N/A && 2.2$\pm$2.9 & 2.4$\pm$3.9 & 8.8$\pm$15.4 & N/A && 29.8$\pm$1.1 & 35.4$\pm$5.0 & 71.6$\pm$3.8 \\
\midrule
RANDOM    & 122.3$\pm$184.3 & 121.9$\pm$186.2 & 18.8$\pm$2.5 & 18.8$\pm$2.6 & 27.2$\pm$14.3 & 10078.2$\pm$10770.5 & 9446.6$\pm$9296.3 & 25.4$\pm$3.2 & 33.0$\pm$4.5 & 65.8$\pm$5.9   \\

\toprule
& \multicolumn{2}{l}{SF1} & & & & \multicolumn{2}{l}{SF4} & &  \\ \midrule
& \multicolumn{1}{l}{Val $\widehat{Q} - \widehat{Q}^*$} & & \multicolumn{1}{l}{SHD} & \multicolumn{1}{l}{SHD-CPDAG} & \multicolumn{1}{l}{SID}  & \multicolumn{1}{l}{Val $\widehat{Q} - \widehat{Q}^*$} & & \multicolumn{1}{l}{SHD} & \multicolumn{1}{l}{SHD-CPDAG} & \multicolumn{1}{l}{SID}  \\ \midrule
$\plrelax$ & -0.7$\pm$0.3 & -0.7$\pm$0.3 & 2.2$\pm$1.5 & 2.4$\pm$1.5 & 2.6$\pm$2.2 & -1.3$\pm$1.4 & -1.3$\pm$1.6 & 8.2$\pm$3.1 & 8.8$\pm$3.3 & 15.4$\pm$5.9 \\
SINKHORN$_{ECP}$     & 0.6$\pm$2.9 & 0.6$\pm$3.0 & 2.8$\pm$3.2 & 3.0$\pm$3.2 & 7.6$\pm$12.3 & -0.3$\pm$4.0 & -0.2$\pm$4.3 & 6.6$\pm$1.5 & 7.0$\pm$1.8 & 11.8$\pm$4.0   \\
URS$_{ECP}$    & 1.6$\pm$1.8 & 1.6$\pm$1.8 & 5.0$\pm$1.7 & 5.4$\pm$2.2 & 7.0$\pm$2.2 & 2.1$\pm$2.3 & 2.3$\pm$2.5 & 12.8$\pm$2.5 & 13.4$\pm$2.2 & 24.8$\pm$5.6   \\
SINKHORN & 22.6$\pm$22.4 & 23.3$\pm$24.5 & 9.0$\pm$0.0 & 9.2$\pm$0.4 & 17.4$\pm$3.8 & 232.4$\pm$251.8 & 240.1$\pm$260.4 & 17.2$\pm$2.8 & 17.6$\pm$3.4 & 34.4$\pm$9.9 \\
URS & 10.1$\pm$5.2 & 10.2$\pm$5.7 & 9.6$\pm$1.2 & 9.6$\pm$1.2 & 14.6$\pm$2.1 & 69.6$\pm$81.6 & 69.9$\pm$79.3 & 14.6$\pm$1.4 & 14.6$\pm$1.2 & 29.2$\pm$5.8 \\
GREEDY-SP & N/A && 0.8$\pm$0.4 & 0.0$\pm$0.0 & 2.8$\pm$1.6 & N/A && 5.0$\pm$9.5 & 4.2$\pm$9.4 & 11.0$\pm$12.7 \\
\midrule
RANDOM    & 35.4$\pm$22.4 & 36.8$\pm$24.7 & 17.2$\pm$2.6 & 18.0$\pm$2.6 & 23.6$\pm$6.4 & 240.3$\pm$251.0 & 248.5$\pm$260.1 & 34.4$\pm$2.6 & 35.6$\pm$2.2 & 31.4$\pm$11.2   \\
\bottomrule
\end{tabular}
}

\caption{Results for ER and SF graphs of 20 nodes}
\label{20nodes}
\resizebox{\textwidth}{!}{
\begin{tabular}{llHllllHlll}
\toprule
& \multicolumn{2}{l}{ER1} & & & & \multicolumn{2}{l}{ER4} & &  \\ \midrule
& \multicolumn{1}{l}{Val $\widehat{Q} - \widehat{Q}^*$} & & \multicolumn{1}{l}{SHD} & \multicolumn{1}{l}{SHD-CPDAG} & \multicolumn{1}{l}{SID}  & \multicolumn{1}{l}{Val $\widehat{Q} - \widehat{Q}^*$} & & \multicolumn{1}{l}{SHD} & \multicolumn{1}{l}{SHD-CPDAG} & \multicolumn{1}{l}{SID}  \\ \midrule
$\plrelax$ & 15.7$\pm$27.3 & 15.8$\pm$27.2 & 14.4$\pm$5.3 & 16.0$\pm$6.2 & 61.0$\pm$48.7 & 468.8$\pm$208.4 & 475.2$\pm$211.7 & 71.0$\pm$5.9 & 72.6$\pm$3.9 & 289.6$\pm$9.1 \\
SINKHORN$_{ECP}$ & 10.4$\pm$8.7 & 10.5$\pm$8.7 & 15.8$\pm$4.7 & 17.0$\pm$6.0 & 84.8$\pm$56.3 & 2519.0$\pm$3715.2 & 2402.0$\pm$3448.9 & 78.0$\pm$6.1 & 78.8$\pm$5.5 & 302.2$\pm$15.8   \\
URS$_{ECP}$    & 27.5$\pm$34.2 & 27.8$\pm$34.7 & 20.6$\pm$6.3 & 21.4$\pm$7.2 & 96.8$\pm$74.6 & 1011.4$\pm$745.5 & 1022.2$\pm$762.4 & 75.8$\pm$2.9 & 76.6$\pm$2.9 & 300.2$\pm$20.3   \\
SINKHORN & 1651.2$\pm$3050.4 & 1654.6$\pm$3042.9 & 24.0$\pm$6.1 & 25.0$\pm$6.7 & 131.2$\pm$76.5 & 126284.6$\pm$194386.3 & 130166.8$\pm$202367.9 & 88.8$\pm$6.0 & 91.0$\pm$5.7 & 330.0$\pm$14.1 \\
URS & 1189.4$\pm$1815.5 & 1173.5$\pm$1753.0 & 26.4$\pm$8.4 & 26.6$\pm$8.6 & 134.2$\pm$75.0 & 7179677.6$\pm$7874489.3 & 7326627.3$\pm$8157542.1 & 93.0$\pm$3.8 & 94.4$\pm$4.5 & 328.0$\pm$11.5 \\
GREEDY-SP & N/A && 18.6$\pm$13.5 & 18.0$\pm$16.6 & 74.0$\pm$53.5 & N/A && 103.4$\pm$10.9 & 105.6$\pm$10.5 & 288.6$\pm$14.7 \\
\midrule
RANDOM    & 895.1$\pm$1270.3 & 905.2$\pm$1264.9 & 37.8$\pm$5.2 & 38.8$\pm$4.9 & 146.8$\pm$79.9 & 109891.2$\pm$74968.7 & 110237.3$\pm$74814.3 & 113.0$\pm$4.9 & 114.4$\pm$4.1 & 330.6$\pm$9.2   \\
\toprule
& \multicolumn{2}{l}{SF1} & & & & \multicolumn{2}{l}{SF4} & &  \\ \midrule
& \multicolumn{1}{l}{Val $\widehat{Q} - \widehat{Q}^*$} & & \multicolumn{1}{l}{SHD} & \multicolumn{1}{l}{SHD-CPDAG} & \multicolumn{1}{l}{SID}  & \multicolumn{1}{l}{Val $\widehat{Q} - \widehat{Q}^*$} & & \multicolumn{1}{l}{SHD} & \multicolumn{1}{l}{SHD-CPDAG} & \multicolumn{1}{l}{SID}  \\ \midrule
$\plrelax$ & -1.5$\pm$0.2 & -1.5$\pm$0.3 & 4.0$\pm$0.6 & 4.6$\pm$0.5 & 4.2$\pm$0.7 & -5.8$\pm$1.2 & -5.8$\pm$0.7 & 20.0$\pm$4.3 & 20.0$\pm$4.1 & 48.4$\pm$16.2 \\
SINKHORN$_{ECP}$     & 1.9$\pm$4.3 & 1.8$\pm$4.0 & 6.6$\pm$2.2 & 6.6$\pm$2.4 & 10.4$\pm$5.0 & -0.4$\pm$2.4 & -0.8$\pm$2.2 & 25.6$\pm$5.6 & 25.8$\pm$5.9 & 58.6$\pm$19.7   \\
URS$_{ECP}$    & 3.0$\pm$2.0 & 3.1$\pm$2.1 & 10.6$\pm$2.0 & 10.6$\pm$1.6 & 14.4$\pm$4.0 & 8.5$\pm$11.8 & 8.1$\pm$11.7 & 30.2$\pm$5.8 & 30.6$\pm$5.2 & 72.2$\pm$25.0   \\
SINKHORN & 38.3$\pm$26.2 & 35.9$\pm$23.1 & 19.0$\pm$0.0 & 19.0$\pm$0.0 & 35.0$\pm$2.4 & 158.2$\pm$99.9 & 158.5$\pm$104.9 & 44.6$\pm$5.8 & 44.8$\pm$6.1 & 103.6$\pm$20.8 \\
URS & 38.3$\pm$26.2 & 35.9$\pm$23.1 & 19.0$\pm$0.0 & 19.0$\pm$0.0 & 35.0$\pm$2.4 & 140.7$\pm$140.6 & 139.4$\pm$143.2 & 42.0$\pm$5.4 & 42.8$\pm$5.1 & 89.8$\pm$20.4 \\
GREEDY-SP & N/A && 2.0$\pm$1.4 & 0.0$\pm$0.0 & 7.0$\pm$5.1 & N/A && 50.6$\pm$31.5 & 49.8$\pm$32.3 & 69.0$\pm$43.2 \\
\midrule
RANDOM    & 94.0$\pm$36.4 & 90.7$\pm$33.8 & 36.2$\pm$2.6 & 36.6$\pm$2.3 & 48.6$\pm$14.7 & 635.5$\pm$182.6 & 627.8$\pm$164.5 & 98.2$\pm$6.1 & 99.2$\pm$5.5 & 168.8$\pm$29.6   \\
\bottomrule
\end{tabular}
}

\caption{Results for ER and SF graphs of 50 nodes}
\resizebox{\textwidth}{!}{
\label{50nodes}
\begin{tabular}{llHllllHlll}
\toprule
& \multicolumn{2}{l}{ER1} & & & & \multicolumn{2}{l}{ER4} & & \\ \midrule
& \multicolumn{1}{l}{Val $\widehat{Q} - \widehat{Q}^*$} & & \multicolumn{1}{l}{SHD} & \multicolumn{1}{l}{SHD-CPDAG} & \multicolumn{1}{l}{SID}  & \multicolumn{1}{l}{Val $\widehat{Q} - \widehat{Q}^*$} & & \multicolumn{1}{l}{SHD} & \multicolumn{1}{l}{SHD-CPDAG} & \multicolumn{1}{l}{SID}  \\ \midrule
$\plrelax$ & -1.8$\pm$1.3 & -2.0$\pm$1.3 & 19.2$\pm$6.9 & 20.6$\pm$7.8 & 103.2$\pm$55.5 & 1863.1$\pm$1703.2 & 1863.3$\pm$1773.8 & 220.6$\pm$42.8 & 221.4$\pm$43.5 & 1779.6$\pm$193.1 \\
SINKHORN$_{ECP}$     & 5.5$\pm$7.0 & 5.7$\pm$6.9 & 30.0$\pm$6.3 & 30.8$\pm$5.8 & 151.8$\pm$35.1 & 43463.9$\pm$70904.3 & 44395.2$\pm$73974.4 & 221.0$\pm$14.7 & 223.2$\pm$15.2 & 1846.4$\pm$158.3   \\
URS$_{ECP}$    & 10.3$\pm$4.7 & 10.5$\pm$5.0 & 41.0$\pm$2.4 & 40.0$\pm$2.7 & 177.6$\pm$17.1 & 22997.9$\pm$38346.1 & 22220.3$\pm$37322.8 & 239.4$\pm$31.6 & 240.2$\pm$31.5 & 1789.8$\pm$154.4  \\
SINKHORN & 90.3$\pm$35.8 & 91.7$\pm$33.8 & 49.6$\pm$4.3 & 49.6$\pm$4.3 & 275.0$\pm$42.5 & 231304.8$\pm$290019.0 & 234554.2$\pm$297849.2 & 248.6$\pm$18.5 & 250.4$\pm$19.1 & 1966.8$\pm$135.5 \\
URS & 90.3$\pm$35.8 & 91.7$\pm$33.8 & 49.6$\pm$4.3 & 49.6$\pm$4.3 & 275.0$\pm$42.5 & 546793216.7$\pm$984510739.7 & 555690870.2$\pm$1010508304.0 & 320.2$\pm$26.8 & 320.8$\pm$27.1 & 2119.0$\pm$130.5 \\
GREEDY-SP & N/A && 38.2$\pm$21.6 & 38.2$\pm$24.6 & 151.6$\pm$84.3 & N/A && 525.6$\pm$35.5 & 526.8$\pm$34.7 & 1951.4$\pm$50.3 \\
\midrule
RANDOM    & 271.0$\pm$71.6 & 277.1$\pm$69.6 & 99.4$\pm$9.3 & 99.8$\pm$9.5 & 301.2$\pm$60.4 & 477442.0$\pm$661243.9 & 490380.3$\pm$689358.0 & 360.8$\pm$23.5 & 361.0$\pm$23.2 & 2175.0$\pm$52.6   \\

\toprule
& \multicolumn{2}{l}{SF1} & & & & \multicolumn{2}{l}{SF4} & &  \\ \midrule
& \multicolumn{1}{l}{Val $\widehat{Q} - \widehat{Q}^*$} & & \multicolumn{1}{l}{SHD} & \multicolumn{1}{l}{SHD-CPDAG} & \multicolumn{1}{l}{SID}  & \multicolumn{1}{l}{Val $\widehat{Q} - \widehat{Q}^*$} & & \multicolumn{1}{l}{SHD} & \multicolumn{1}{l}{SHD-CPDAG} & \multicolumn{1}{l}{SID}  \\ \midrule
$\plrelax$ & -3.9$\pm$0.5 & -4.0$\pm$0.3 & 11.4$\pm$3.3 & 11.8$\pm$2.9 & 14.4$\pm$2.7 & -1.1$\pm$7.6 & -2.5$\pm$6.7 & 70.0$\pm$9.9 & 70.6$\pm$11.2 & 219.0$\pm$20.3 \\
SINKHORN$_{ECP}$     & 25.1$\pm$18.2 & 25.6$\pm$20.4 & 28.6$\pm$6.5 & 28.4$\pm$6.1 & 58.4$\pm$12.1 & 124.3$\pm$126.0 & 123.0$\pm$126.0 & 94.4$\pm$22.7 & 95.6$\pm$23.0 & 257.2$\pm$25.8   \\
URS$_{ECP}$    & 32.1$\pm$44.3 & 30.7$\pm$43.1 & 33.4$\pm$10.2 & 33.6$\pm$10.7 & 55.6$\pm$32.7 & 164.4$\pm$53.1 & 162.5$\pm$56.1 & 110.6$\pm$12.8 & 111.4$\pm$13.7 & 319.6$\pm$18.1   \\
SINKHORN & 138.2$\pm$68.2 & 132.5$\pm$65.6 & 49.0$\pm$0.0 & 49.0$\pm$0.0 & 110.6$\pm$5.5 & 10238.2$\pm$15850.1 & 10549.4$\pm$16302.8 & 139.0$\pm$8.3 & 139.6$\pm$8.1 & 387.0$\pm$37.2 \\
URS & 138.2$\pm$68.2 & 132.5$\pm$65.6 & 49.0$\pm$0.0 & 49.0$\pm$0.0 & 110.6$\pm$5.5 & 7966.9$\pm$4838.0 & 8191.0$\pm$5088.9 & 142.8$\pm$11.8 & 144.2$\pm$12.1 & 527.4$\pm$86.8 \\
GREEDY-SP & N/A && 38.8$\pm$39.3 & 35.4$\pm$39.6 & 54.8$\pm$20.6 & N/A && 381.2$\pm$76.2 & 384.2$\pm$77.0 & 963.0$\pm$475.7 \\
\midrule
RANDOM    & 380.1$\pm$207.8 & 368.4$\pm$200.7 & 97.8$\pm$7.3 & 97.8$\pm$7.3 & 155.4$\pm$31.2 & 10109.8$\pm$2027.0 & 10363.0$\pm$1889.2 & 312.0$\pm$14.9 & 312.4$\pm$15.0 & 807.0$\pm$101.7   \\
\bottomrule
\end{tabular}
}
\end{table*}

\subsection{Causal Structure Learning Through Order Search}
Directed acyclic graph (DAG) models are popular tools for describing causal relationships and for guiding attempts to learn them from data. Learning the structure of a DAG remains challenging because of the combinatorial acyclicity constraint. A common way to model causal relations is a structural equation model (SEM). Let $X$ be $\numitems$-dimensional random variable, then relations are described as follows:
\begin{equation}\label{sem}
    X_i = f_i(X_{pa(i)}, \varepsilon_i),
\end{equation}
where $pa(i)$ is the set of parent vertices of variable $X_i$ and $\varepsilon_i$ is independent noise. Edge set $\{\cup_{i=1}^{\numitems} \cup_{j \in pa(i)} j \xrightarrow{} i\}$ describes DAG $G$ on $\numitems$ vertices associated with joint distribution $\mathbb{P}_G(X) = \prod_{i=1}^{\numitems} \mathbb{P}(X_i | pa(X_i))$. The basic structure learning problem can therefore be formulated as follows: let $\bm{X}$ be data matrix consisting of $n$ i.i.d. samples of random variable $X$. Also let $\mathbb{D}$ be space of DAGs. Then, given observations $\bm{X}$ the task is to find DAG $G \in \mathbb{D}$ or so-called Bayesian Network for joint distribution $\mathbb{P}(X)$:
\begin{equation}\label{finddagtask}
    \min\limits_{G \in \mathbb{D}} Q(G, \bm{X})
\end{equation}
where $Q$ is function that scores DAG $G$ given data.

To incorporate permutations in the objective (\ref{finddagtask}) we consider parametrization of DAG adjacency matrix using nilpotent matrices which are upper triangular in basis induced by topological ordering, namely $W_G = P A P^T$ where $A$ is strictly upper triangular adjacency matrix which describes parent sets of variables and permutation matrix $P$ which describes topological ordering. Then optimization over DAGs can thus be seen as an optimization over topological orderings
\begin{equation}\label{order_score}
    \min\limits_{P \in \mathcal{P}_\numitems} \widehat{Q}(P, \bm{X}),
\end{equation}
where $\widehat{Q}$ scores topological ordering $P$ and $\mathcal{P}_\numitems$ is the set of permutation matrices of size $\numitems$. Optimization over $A$ is usually hidden in the computation of $\widehat{Q}$. It is worth noting that this approach is similar to order MCMC \cite{friedman2003}, however our work considers gradient-based optimization over permutations matrices rather than discrete order changes.

\subsubsection{Continuous data} We consider linear additive noise SEMs:
\begin{equation}\label{linearsem}
    X = W^T X + \varepsilon
\end{equation}
where $W = P A P^T$ and non-zero elements of $A$ describe linear coefficients and parent sets for each variable $X_i$.

As score function $\widehat{Q}$ we take regularized mean squared loss combined with sparsity-inducing $L1$ regularization term
\begin{equation}\label{plscorecontinuous}
    \widehat{Q}(P, \bm{X}) = \min\limits_{A \in \mathbb{A}} \frac{1}{2n} \| \bm{X} - P A P^T \bm{X}\|^2_F + \lambda \|\tovec(A)\|_1,
\end{equation}
where $\mathbb{A}$ is the set of strictly upper triangular matrices. Computing $\widehat{Q}$ itself involves optimization problem, which can be efficiently solved using accelerated proximal gradient for convex composite function optimization \cite{nesterov2013gradient}.
To apply the proposed method, we reformulate (\ref{order_score}) as variational optimization with respect to parameters of a Plackett-Luce distribution:
\begin{equation}
\label{eqn:relaxdag}
    \min\limits_{\theta} \mathop{\mathbb{E}}\limits_{p(b|\theta)} \widehat{Q}(P_b, \bm{X})
\end{equation}
where $p(b|\theta)=$ Plackett-Luce$(b|\theta)$, and $P_b$ is a permutation matrix with $p_{i, b_i} = 1$. For variational optimization, we only apply $\plrelax$ and treat $\widehat{Q}(P, \bm{X})$ as a black-box function to avoid unrolling the optimizer to compute gradients.

As a concurrent approach, we consider work by \citet{sinkhorn} which proposes relaxing optimization over a set of permutations to a set of doubly-stochastic matrices using the Sinkhorn operator. Another recent work by \citet{urs} proposes relaxation to the set of unimodal row-stochastic matrices (URS) which intersects the set of doubly-stochastic matrices and contains the set of all permutation matrices. Since these methods can't be used to optimize black-box functions we reformulate (\ref{eqn:relaxdag}) as:
\begin{equation}
\label{eqn:phidag}
    \min\limits_{\phi} \min\limits_{A \in \mathbb{A}} \frac{1}{2n} \| \bm{X} - P(\phi) A P(\phi)^T \bm{X}\|^2_F + \lambda \|\tovec(A)\|_1
\end{equation}
where $\phi$ are the parameters of the corresponding relaxation. We optimize (\ref{eqn:phidag}) coordinate-wise using gradient descent with respect to $\phi$ and accelerated proximal gradient optimization with respect to $A$. We refer to the optimization of this objective as SINKHORN or URS according to the used relaxation.

We also try an alternative approach for the above relaxations. Since $P(\phi)$ is not a permutation matrix during training we extend (\ref{eqn:phidag}) with an orthogonality constraint and replace $\|vec(A)\|$ with $H_{\mu}(vec(P A P^T))$ where $H_{\mu}$ is the Huber relaxation of $L1$ norm and $\mu$ is a hyperparameter controlling tightness of relaxation:
\begin{align}\label{eqn:phidagaugmented}
    \begin{aligned}
    & \min\limits_{\phi} \min\limits_{A \in \mathbb{A}} & & \frac{1}{2n} \| \bm{X} - P(\phi) A P(\phi)^T \bm{X}\|^2_F + \\
    & & & + \lambda H_\mu(\tovec(P(\phi) A P^T(\phi))) \\
    & \subjectto & & \|P(\phi) P^T(\phi) - I_\numitems\|^2_F = 0 \\
    \end{aligned}
\end{align}
We use an Augmented Lagrangian \cite{augmentedlagrangian} to solve this equality constrained optimization problem (ECP) (\ref{eqn:phidagaugmented}) and refer to the solutions as SINKHORN$_{ECP}$ or URS$_{ECP}$ correspondingly.

We simulated graphs from two well-known random graph models with different degree distributions: Erdos-Renyi random graphs and Scale-free networks with \textit{\numitems} and \textit{4 \numitems} expected number of edges, denoted by ER1, ER4, SF1, SF4 respectively. Given a random acyclic graph we assigned edge weights independently from $U([-2; -0.5] \cup [0.5; 2])$ to obtain weight matrix $W$. To generate data matrix $\bm{X}$ we follow generating process of linear SEM (\ref{linearsem}) with standard Gaussian noise.

As a sanity check, we also introduce a simple baseline. We generate Erdos-Renyi random graphs with the corresponding expected number of edges and refer to it as RANDOM baseline. For comparison, we also include the Greedy Sparse Permutation (Greedy-SP) algorithm \cite{GreedySP}. This algorithm casts DAG structure learning as a linear programming problem with graph sparsity as the linear objective function, and a sub-polytope of the permutohedron as the feasible region. Whilst this algorithm also searches permutations as a proxy to DAGs to reduce the size of the search space, it is in essence a constraint-based method - rather than optimising a DAG score function, it searches for the sparsest DAG which satisfies the conditional independence relations found. Conversely, our gradient-based method does not rely on these conditional independence tests, which typically require the simplifying assumptions of CI tests, and is able to use linear as well as non-linear objective functions (e.g. in the discrete data experiment, the quotient normalized maximum likelihood score is non-linear and non-differentiable).

For each method we report the score difference $\widehat{Q}$ (\ref{plscorecontinuous}) between learned and ground truth DAGs on additionally generated validation samples $\bm{X}_{val}$, as well as three DAG metrics from causal inference literature. The quoted score difference shows the effectiveness of our method for optimizing the chosen score function, while the DAG metrics show how well it performs on the problem itself. \textit{Structural hamming distance} (SHD) is the number of edge additions, removals, and reversals required to get from the learned structure to the ground truth. Multiple DAGs can represent the same set of conditional independence relations, forming a Markov equivalence class; this can be represented by a completed partially directed acyclic graph (CPDAG). We also report SHD-CPDAG - the SHD between the CPDAG the learned structure belongs to and that of the true structure. \textit{Structural interventional distance} (SID) \cite{sid} quantifies the distance between two DAGs in terms of their respective causal inference statements. This gives an indication of accuracy of computed interventions using the learned graph.

We consider graphs of 10, 20 and 50 nodes. 

For $\plrelax$ we take the mode of the distribution after training. For SINKHORN relaxation we apply the Hungarian algorithm to find the closest permutation matrix. For URS we use the argmax permutation property to obtain the permutation matrix. Regularization 
coefficient $\lambda$ is set to $0.5$ for all methods. 

Tables \ref{10nodes}-\ref{50nodes} show the performance of all methods for varying number of nodes $\numitems$ averaged across 5 random seeds (the error ranges represent standard deviation). We can see that the proposed method outperforms baselines in the majority of settings. Also, it is worth mentioning that SINKHORN and URS perform poorly in terms of score function values due to the fact that the optimization is carried out over the set of relaxed matrices. This leads to deterioration in score value $\widehat{Q}$ when relaxation is transformed to permutation. As we can see there is no such problem with ECP versions of relaxations, though they perform worse than $\plrelax$ and require additional constrained optimization techniques to be applied. Also, one more observation should be explained: $\plrelax$ almost always ends up with better solutions in terms of score function than the ground truth DAG, therefore solves the optimization problem well. However, it is not ideal in terms of metrics. \citet{identifiability} proved that given enough data, it is possible to identify the ground truth DAG if data was generated from linear SEM with Gaussian homogeneous noise. Authors used $L0$-regularized mean squared error score function, but it is non-convex and hard to optimize, therefore $L1$-regularization is used in practice. Because of relaxation of the $L0$ norm and finite amount of data all guarantees vanish, and we observe inconsistency between the metrics of interest and values of the surrogate score function $\widehat{Q}$.

\subsubsection{Discrete data}
Due to the discrete and nonlinear nature of categorical data, it cannot be modeled with the SEM defined previously. Discrete variable networks can however be modeled as generated by sampling each node's conditional probability table, depending only on the configuration of its parent nodes. In the standard general form this is $X_i = f_i(X_{pa(i)})$, where $f_i$ is assumed to be multinomial, thus 
$$ f_i(X_{pa(i)}) \sim Multinomial(\Theta_{X_{i}} | Pa(X_{i})) $$

where $\Theta_{X_{i}} | Pa(X_{i})$ are the conditional probabilities $\theta_{i,j,k} = P(X_i = k | Pa(X_{i}) = j)$. 

Rather than learning the optimal $A$ for a given $P$ by minimising a training loss, we can therefore instead try to maximise the marginal likelihood based on the above model
\begin{equation}
    Q(P, \textbf{X}) = \max \limits_{A \in \mathbb{A}} P(\textbf{X} | A, P)
\end{equation}
which can be found using the factorisation
\begin{equation}
\label{eqn:score_decomposition}
    P(\textbf{X}|A,P) = \prod\limits_{i=1}^{d} \prod\limits_{j=1}^{q_i} P(\textbf{X}_{i, pa(i)=j};\alpha).
\end{equation}
As a result of the decomposition of the score by node in equation (\ref{eqn:score_decomposition}), the \textit{maximum a posteriori} (MAP) parent set can be selected from the set of parents permitted by the topological ordering for each node, independently of the rest. Due to the ordering, the graph resulting from combining each of these MAP parent connections is guaranteed to be acyclic, thus the exact MAP DAG for a given ordering can be found. Due to the combinatorial size of even this reduced search space, the set of permitted parents for a given node is reduced further, to only those that cannot be easily proven to be conditionally independent - as determined by a standard constraint-based method (in this case the PC-stable algorithm \cite{pcstable}). As this finds the exact solution for a reduced search space, the result is an approximation of the best score possible for the ordering. Whilst this provides an approximate score for any given order, it is a non-differentiable black-box function; therefore whilst our method can be applied to this permutation optimization, options are severely limited - the SINKHORN and URS methods used for continuous SEM graph benchmarks for example cannot be used. For a simple evaluation, Table \ref{tab:alarm_result} shows the result of our method on data sampled from the standard ALARM network compared against random orders, and permutations optimized by order MCMC \cite{friedman2003}, all using the same MAP DAG method described above, maximizing the quotient normalized maximum likelihood score \cite{qnml}. Higher Val $\widehat{Q} - \widehat{Q}^*$ is better, other metrics lower is better. Whilst Table \ref{tab:alarm_result} shows our algorithm to be less effective than MCMC for this task, the comparison is not particularly favorable - MCMC is performed directly on permutations, rather than attempting to learn the Plackett-Luce distribution over permutations - thus the MCMC simply attempts to find a good local minimum in the score space. To give a lower bound to performance, we also compare to the MAP DAGs of 1000 random permutations, computed in the same way as for MCMC and our algorithm, showing sampling the learned Plackett-Luce distribution gives permutations far better than random. 

\begin{table}
\caption{Results for ALARM graph (37 nodes)}
\label{tab:alarm_result}
\resizebox{.95\columnwidth}{!}{
\begin{tabular}{lcccc}
\toprule
         & Val $\widehat{Q} - \widehat{Q}^*$  & SHD      & \small{SHD-CPDAG} & SID        \\ \midrule
\small{$\plrelax$} &  -15645.2$\pm$3255.8  & 14.6$\pm$1.7 & 19.0$\pm$2.3 & 214.2$\pm$31.8        \\
\small{SINKHORN}   & \multicolumn{4}{c}{N/A}           \\
\small{URS}        & \multicolumn{4}{c}{N/A}           \\ 
\small{ORDER MCMC}  & -13404.7$\pm$2224.6 & 8.6$\pm$1.1     & 10.6$\pm$0.5  & 104.4$\pm$20.8      \\ \midrule
RANDOM   & -75022.7$\pm$9647.7   & 25.8$\pm$3.7 & 30.0$\pm$3.9  & 478.8$\pm$70.8 \\ \bottomrule
\end{tabular}
}
\end{table}

\section{Conclusion}
In this work we proposed a gradient-based optimization method, with unique capabilities for application to Plackett-Luce distributions over permutations. A proof of concept experiment shows our method outperforms existing methods for differentiable objective functions, whilst also generalizing to non-differentiable black-box functions, and being applicable to permutation learning despite the factorial complexity. This allowed us to extend Plackett-Luce distribution based causal graphical model structure learning beyond the simple SEM based methods, to the more general case of DAGs of arbitrary variable types.

In future, our method could be combined with other standard scoring functions from Bayesian network literature - providing they decompose as described in equation (\ref{eqn:score_decomposition}) - for DAG structure learning of continuous data from different model types. Other potential applications include approximate inference for probabilistic models with latent permutations, routing problems and combinatorial problems for permutations.

\section*{Acknowledgements}
This work was partly supported by Sberbank AI Lab and UK EPSRC project EP/P03442X/1. Kirill Struminsky proved proposition~\ref{plconditionalreparametrization} and was supported by the Russian Science Foundation grant no.~19-71-30020. The authors thank NRU HSE for providing computational resources, NVIDIA for GPU donations, and Amazon for AWS Cloud Credits.

\section{Appendix} \label{plderivationssection}
We prove Proposition~\ref{plconditionalreparametrization} in this section. We first discuss the properties of Gumbel distribution. Then we discuss the generative processes for the densities used for $p(\gumbelrv \mid \plrv, \score)$ in Eq.~\ref{conditionalgumbelsampler}.
Then we show that $p(\plrv \mid \gumbelrv) p(\gumbelrv \mid \score) = p(\plrv \mid \score) p(\gumbelrv \mid \plrv, \score)$ for the unconditional Gumbel density $p(\gumbelrv \mid \score)$ and the Plackett-Luce distribution $p(\plrv \mid \score)$.
\subsection{Density for the Gumbel distribution and the truncated Gumbel distribution}
The density function of the Gumbel distribution with location parameter $\mu$ is
\begin{equation}
\gumbelpdf{\mu}(\gumbelrv) = \exp(-\gumbelrv + \mu) \exp(-\exp(-\gumbelrv + \mu))
\end{equation}
and the cumulative density function is
\begin{equation}
\gumbelcdf{\mu} = \exp(-\exp(-\gumbelrv + \mu)).
\end{equation}
Our derivation of the conditional distribution $p(\plrv \mid \gumbelrv, \score)$ relies on the additive property of the cumulative density function of the Gumbel distribution
\begin{align}
& \gumbelcdf{\log(\exp\mu + \exp\nu)}(\gumbelrv) = \nonumber\\ & \exp(-\exp(\gumbelrv)(\exp \mu + \exp \nu)) = \gumbelcdf{\mu}(\gumbelrv) \gumbelcdf{\nu}(\gumbelrv),
\end{align}
which we enfold in the following auxiliary claim.
\begin{lemma}\label{appendixprooflemma}
For permutation $\plrv \in \permutations$, score vector $\score \in \mathbb R^\numitems$ and $i=1,\dots,\numitems$ and the argument vector $\gumbelrv \in \mathbb R^\numitems$ we have
\begin{align}
& \gumbelpdf{\score_{\plrv_i}} (\gumbelrv_{\plrv_i}) \gumbelcdf{\log(\sum_{j=i + 1}^\numitems \exp \score_{\plrv_{j}})} (\gumbelrv_{\plrv_i}) \\
& = \frac{\exp \score_{\plrv_i}}{\sum_{j=i}^\numitems \exp \score_{\plrv_j}} \gumbelpdf{\log(\sum_{j=i}^\numitems \exp \score_{\plrv_{j}})} (\gumbelrv_{\plrv_i}).
\end{align}
\end{lemma}

\begin{proof}
For brevity, we denote $\exp \score_i$ as $p_i$. We then rewrite the density $\gumbelpdf{\log p_{\plrv_i}}(\gumbelrv_{\plrv_i})$ through the exponent $\exp(-\gumbelrv_{\plrv_i} + \log p_{\plrv_i})$ and c.d.f. $\gumbelcdf{\log p_{\plrv_{i}}}(\gumbelrv_{\plrv_i})$ and apply the additive property in Eq.~\ref{lemmaadditivepropertytransition}:
\begin{align}
& \gumbelpdf{\log p_{\plrv_i}} (\gumbelrv_{\plrv_i}) \gumbelcdf{\log(\sum_{j=i + 1}^\numitems p_{\plrv_{j}})} (\gumbelrv_{\plrv_i}) \\
& =  p_{\plrv_i} \exp(-\gumbelrv_{\plrv_i}) \gumbelcdf{\log p_{\plrv_{i}}} (\gumbelrv_{\plrv_i}) \gumbelcdf{\log(\sum_{j=i + 1}^\numitems p_{\plrv_{j}})} (\gumbelrv_{\plrv_i}) \\
& = p_{\plrv_i} \exp(-\gumbelrv_{\plrv_i}) \gumbelcdf{\log(\sum_{j=i}^\numitems p_{\plrv_{j}})} (\gumbelrv_{\plrv_i}) \label{lemmaadditivepropertytransition}\\
& = p_{\plrv_i} \frac{\sum_{j=i}^\numitems p_{\plrv_j}}{\sum_{j=i}^\numitems p_{\plrv_j}} \exp(-\gumbelrv_{\plrv_i})
\gumbelcdf{\log(\sum_{j=i}^\numitems p_{\plrv_j})} (\gumbelrv_{\plrv_i}) \\
& = \frac{p_{\plrv_i}}{\sum_{j=1}^\numitems p_{\plrv_j}} \gumbelpdf{\log( \sum_{j=i}^\numitems p_{\plrv_j})} (\gumbelrv_{\plrv_i}).
\end{align}
The last step collapses the exponent and the c.d.f. into the density function $\gumbelpdf{\log( \sum_{j=i}^\numitems p_{\plrv_j})} (\gumbelrv_{\plrv_i})$.
\end{proof}

Finally, to define the density of conditional distribution $p(\plrv \mid \gumbelrv, \score)$ we define the density of the truncated Gumbel distribution $\gumbelpdf{\mu}^{z_0} (\gumbelrv) \propto \gumbelpdf{\mu} (\gumbelrv) I[\gumbelrv \leq z_0]$:
\begin{equation}
\gumbelpdf{\mu}^{z_0} (\gumbelrv) = \frac{\gumbelpdf{\mu} (\gumbelrv)}{\gumbelcdf{\mu} (z_0)} (\gumbelrv) I[\gumbelrv \leq z_0],
\end{equation}
where the superscript $z_0$ denotes the truncation parameter.

\subsection{Reparametrization for the Gumbel distribution and the truncated Gumbel distribution}
The reparametrization trick requires representing a draw from a distribution as a deterministic transformation of a fixed distribution sample and a distribution parameter. For a sample $\gumbelrv$ from the Gumbel distribution $\gumbel(\mu, 1)$ with location parameter $\mu$ the representation is
\begin{equation}
\gumbelrv = \mu - \log(-\log\uniformrv),\ \uniformrv \sim \uniform.
\end{equation}
For the Gumbel distribution truncated at $z_0$ \cite{maddison2014sampling} proposed an analogous representation
\begin{align}
\gumbelrv &= \mu - \log(-\log \uniformrv + \exp(-z_0 + \mu)) \nonumber \\
          &= -\log \left(-\frac{\log\uniformrv}{\exp \mu} + \exp(-z_0) \right) \label{truncatedgumbelreparametrization}\\
& \uniformrv \sim \uniform.
\end{align}
In particular, the sampling schemes in Eq.~\ref{categoricalconditionalgumbelsampler} and Eq.~\ref{conditionalgumbelsampler} generate samples from the truncated Gumbel distribution.

\subsection{The derivation of the conditional distribution}
We now derive the conditional distribution and the sampling scheme defined in Proposition~\ref{plconditionalreparametrization}.

The joint distribution of the permutation $\plrv$ and the Gumbel samples $\gumbelrv$ is
\begin{align}
p(\plrv, \gumbelrv \mid \score) = p(\plrv \mid \gumbelrv) p(\gumbelrv \mid \score) \\ 
= \gumbelpdf{\score_{\plrv_1}}(\gumbelrv_{\plrv_1}) \prod_{i=2}^{\numitems} \left( \gumbelpdf{\score_{\plrv_i}}(\gumbelrv_{\plrv_i}) I[\gumbelrv_{\plrv_{i - 1}} \geq \gumbelrv_{\plrv_{i}}] \right)
\end{align}
We first multiply and divide the joint density by the c.d.f. $\gumbelcdf{\log(\sum_{i=2}^\numitems \exp \score_{\plrv_i})} (\gumbelrv_{\plrv_1})$ and apply Lemma~\ref{appendixprooflemma}
\begin{align}
\frac{\gumbelcdf{\log(\sum_{i=2}^\numitems \exp \score_{\plrv_i})} (\gumbelrv_{\plrv_1})}{\gumbelcdf{\log(\sum_{i=2}^\numitems \exp \score_{\plrv_i})} (\gumbelrv_{\plrv_1})} \gumbelpdf{\score_{\plrv_1}} (\gumbelrv_{\plrv_1}) \prod_{i=2}^\numitems \dots \\
= \frac{\exp \score_{\plrv_1}}{\sum_{i=1}^\numitems \exp \score_{\plrv_i}} \frac{\gumbelpdf{\log(\sum_{i=1}^\numitems \exp \score_{\plrv_i})} (\gumbelrv_{\plrv_1})}{\gumbelcdf{\log(\sum_{i=2}^\numitems \exp \score_{\plrv_i})} (\gumbelrv_{\plrv_1})} \prod_{i=2}^\numitems \dots.
\end{align}
Next, we apply Lemma~\ref{appendixprooflemma} to combine the c.d.f. in the denominator $\gumbelcdf{\log(\sum_{j=i}^\numitems \exp \score_{\plrv_j})} (\gumbelrv_{\plrv_{i-1}})$ and the term $\gumbelpdf{\score_{\plrv_i}}(\gumbelrv_{\plrv_i}) I[\gumbelrv_{\plrv_{i - 1}} \geq \gumbelrv_{\plrv_{i}}]$ inside the product
\begin{align}
& \frac{\gumbelpdf{\score_{\plrv_i}}(\gumbelrv_{\plrv_i}) I[\gumbelrv_{\plrv_{i - 1}} \geq \gumbelrv_{\plrv_{i}}]}{\gumbelcdf{\log(\sum_{j=i}^\numitems \exp \score_{\plrv_j})} (\gumbelrv_{\plrv_{i - 1}})} \\ 
 & = 
\frac{\gumbelpdf{\score_{\plrv_i}}(\gumbelrv_{\plrv_i}) I[\gumbelrv_{\plrv_{i - 1}} \geq \gumbelrv_{\plrv_{i}}]}{\gumbelcdf{\log(\sum_{j=i}^\numitems \exp \score_{\plrv_j})} (\gumbelrv_{\plrv_{i - 1}})}
 \tfrac{\gumbelcdf{\log(\sum_{j=i + 1}^\numitems \exp \score_{\plrv_j})} (\gumbelrv_{\plrv_{i}})}{\gumbelcdf{\log(\sum_{j=i + 1}^\numitems \exp \score_{\plrv_j})} (\gumbelrv_{\plrv_{i}})} \\
& = \tfrac{\exp \score_{\plrv_i}}{\sum_{j=i}^\numitems \exp \score_{\plrv_j}} \frac{\gumbelpdf{\log(\sum_{j=i}^\numitems \exp \score_{\plrv_j})}^{\gumbelrv_{\plrv_{i - 1}}} (\gumbelrv_{\plrv_i})}{\gumbelcdf{\log(\sum_{j=i + 1}^\numitems \exp \score_{\plrv_j})}(\gumbelrv_{\plrv_i})} 
\end{align}
and obtain the truncated distribution $\gumbelpdf{\log(\sum_{j=i}^\numitems \exp \score_{\plrv_j})}^{\gumbelrv_{\plrv_{i - 1}}} (\gumbelrv_{\plrv_i})$ along with one factor of the Plackett-Luce probability $\tfrac{\exp \score_{\plrv_i}}{\sum_{j=i}^\numitems \exp \score_{\plrv_j}}$.
Also, after the transformation the summation index in the denominator c.d.f. changes from $i$ to $i + 1$. This gives us an induction step that we apply sequentially for $i=2,\dots, \numitems - 1$.
For $i = \numitems$ the denominator c.d.f. $\gumbelcdf{\log \exp \score_\numitems}(\gumbelrv_{\plrv_{\numitems - 1}})$ and the product term $\gumbelpdf{\log\exp\score_\numitems}(\gumbelrv_{\plrv_\numitems}) I[\gumbelrv_{\numitems - 1} \geq \gumbelrv_\numitems]$ combine into the truncated Gumbel distribution with density $\gumbelpdf{\log \exp \score_\numitems}^{\gumbelrv_{\plrv_{\numitems - 1}}}(\gumbelrv_{\plrv_\numitems})$.

As a result, we rearrange $p(\plrv, \gumbelrv \mid \score)$ into the product of the truncated Gumbel distribution densities $p(\gumbelrv \mid \plrv, \score)$ and the probability of the Plackett-Luce distribution $p(\plrv \mid \score)$:
\begin{equation}
\prod_{i=1}^\numitems \tfrac{\exp \score_{\plrv_i}}{\sum_{j=i}^\numitems \exp \score_{\plrv_j}}
\left( \gumbelpdf{0} (\gumbelrv_{\plrv_1})
\prod_{i=2}^\numitems \gumbelpdf{\log \sum_{j=i}^\numitems \exp \score_j}^{\gumbelrv_{\plrv_{i - 1}}}(\gumbelrv_{\plrv_i}) \right).
\end{equation}

Finally, to obtain the claim of Proposition~\ref{plconditionalreparametrization} we apply the reparametrized sampling scheme defined in Eq.~\ref{truncatedgumbelreparametrization}.

\bibliography{references}
\bibliographystyle{aaai}
\end{document}